\documentclass[review,10pt]{ReportTemplate}

\usepackage{times}
\usepackage{graphicx}
\usepackage{bm}
\usepackage{paralist,algorithmic,algorithm}
\usepackage{amsmath,mathrsfs,amssymb}
\usepackage{url}            
\usepackage{amsfonts}       
\usepackage{nicefrac}       
\usepackage{microtype}      
\usepackage{booktabs}       
\usepackage{longtable}
\usepackage{threeparttable}
\usepackage{multirow}
\usepackage{soul}
\usepackage{subfigure}
\usepackage{wrapfig}
\usepackage[american]{babel}
\usepackage{algorithm}
\usepackage{algorithmic}
\usepackage{setspace}
\usepackage{stfloats}
\usepackage{lipsum}
\usepackage{multicol}

\newtheorem{thm}{Theorem}
\newtheorem{corollary}{Corollary}

\newenvironment{proof}{\textbf{Proof:}\ }{\hspace{\stretch{1}}$\square$\\}

\begin{document}
\begin{frontmatter}
\title{Bifurcation Spiking Neural Network}

\author{Shao-Qun Zhang}
\author{Zhao-Yu Zhang}
\author{Zhi-Hua Zhou}
\address{National Key Laboratory for Novel Software Technology\\
Nanjing University, Nanjing 210023, China\\~\\
\normalsize \{zhangsq,zhangzhaoyu,zhouzh\}@nju.edu.cn
}

\begin{abstract}
Spiking neural networks (SNNs) have attracted much attention due to their great potential for modeling time-dependent signals. The performance of SNNs depends not only on picking an apposite architecture and searching optimal connection weights as well as conventional deep neural networks, but also on the careful tuning of many hyper-parameters within fundamental spiking neural models. However, so far, there has been less systematic work on analyzing SNNs' dynamical characteristics, especially ones relative to these \emph{internal} hyper-parameters, which leads to whether SNNs are adequate for modeling actual data relies on fortune. In this work, we provide a theoretical framework for investigating spiking neural models from the perspective of dynamical systems. As a result, we point out that the LIF model with control rate hyper-parameters is a \emph{bifurcation} dynamical system. This point explains why the performance of SNNs is so sensitive to the setting of control rate hyper-parameters, leading to a recommendation that diverse and adaptive eigenvalues are beneficial to improve the performance of SNNs. Inspired by this insight, by enabling the spiking dynamical system to have adaptable eigenvalues, we develop the \emph{Bifurcation Spiking Neural Network} (BSNN), which has a laudable performance in supervised SNNs and is insensitive to the setting of control rates. Experiments validate the effectiveness of BSNN on several benchmark data sets, showing that BSNN achieves superior performance to existing SNNs and is robust to the setting of control rates.
\end{abstract}

\end{frontmatter}

\section{Introduction} \label{sec:intro}
Spiking neural networks (SNNs) take into account the time of spike firing rather than simply relying on the accumulated signal strength in conventional neural networks, and thus offering the possibility for modeling time-dependent data~\citep{shimokawa1999,vanrullen2005}. The fundamental spiking neural model is usually formulated as a first-order parabolic equation with many biologically realistic (\emph{i.e.}, internal) hyper-parameters, \emph{e.g.}, the membrane capacitance hyper-parameter in Hodgkin-Huxley (HH) model~\citep{gerstner2002} and the membrane time hyper-parameter in Leaky Integrate-and-Fire (LIF) model~\citep{burkitt2006a,burkitt2006b}. Thus, the performance of SNNs depends not only on determining the neural network architecture and training connection weights as well as conventional deep neural networks but also on the careful tuning of these \emph{internal} hyper-parameters. Existing SNNs often follow the bio-plausible knowledge in neuroscience to set these internal hyper-parameters and then make minor a tuning of some hyper-parameters~\citep{hohn2001,carlson2014}. Usually, the traditional manners look clumsy and have unrobust performance with large computation due to the lack of systematic analysis on designating which hyper-parameters are sensitive and guiding how to tune them.

This work investigates the dynamical properties of the LIF-modeling SNNs, especially the influence of hyper-parameters on the model dynamics. As a result, we declare that the LIF model is a bifurcation dynamical system, which means that its topology depends sensitively on the control rate hyper-parameters. This result derives three essential conclusions: (1) The performance of SNNs is sensitive to the setting of control rates, which is consistent with the facts. (2) It is necessary to enable diverse and learnable control rates, corresponding to the eigenvalues of bifurcation dynamical systems, for achieving adaptive systems. This result argues the conventional manners that the control rates are neatly preset as a negative fixed value. (3) The role of control rates cannot be replaced by learnable connection parameters and other hyper-parameters.

However, training control rates is a very tricky challenge. Since control rates and connection weights are entangled during the training process, the approaches~\citep{hunter2012,lorenzo2017} of turning hyper-parameters in conventional neural networks cannot be directly used to solve this issue. An alternative way is to sample the control rates from a certain pre-defined distribution and find the optimal ones by alternating optimization. Nevertheless, this method usually succeeds on an apposite distribution and larger computation and storage. 

To tackle the challenges above and improve the performance of SNNs, we propose the \emph{Bifurcation Spiking Neural Network} (BSNN). By exploiting the bifurcation theory, we convert the issue of learning a group of adaptive control rates into a new problem of learning a collection of apposite eigenvalues. So BSNN overcomes the obstacle that controls rates interact with connection weights, leading to a robust control rate setting and achieves a laudable performance with considerably less computation and storage than the alternating optimization approaches. The experiments conducted on four benchmark data sets demonstrate the effectiveness of BSNN, showing that its performance surpasses existing SNNs and is robust to the setting of control rates.

Our main contributions are summarized as follows:
\begin{itemize}
	\item We provide a theoretical framework for studying the dynamical properties of spiking neural models, \emph{e.g.}, we show the LIF model is a bifurcation dynamical system in Section~\ref{sec:DS}. 
	\item We point out the fact that the control rate hyper-parameter, rather than other ones, is sensitive to the performance of SNNs with LIF neurons. 
	\item We present the BSNN, which has adaptive eigenvalues, leaving a robust setting of the control rate hyper-parameters, and achieves a better performance than the existing state-of-the-art SNN models.
	\item We perform numerical studies on several data sets to demonstrate the conclusions above and advantages of our proposed models.
\end{itemize}

The rest of this paper is organized as follows. Section~\ref{sec:rw} reviews the related works. Section~\ref{sec:snm} introduces some preliminary knowledge about spiking neural models. Section~\ref{sec:DS} establishes a theoretical analysis for SNNs based on dynamical system theory and provides the alternating optimization approaches for achieving adaptive systems. Section~\ref{sec:BSNN} formally presents the BSNN with a concrete implementation. The experiments are conducted in Section~\ref{sec:Experiments}. Finally, we conclude our work in Section~\ref{sec:Conclusion}.

\section{Related Works} \label{sec:rw}
We first review the approaches for setting and tuning hyper-parameters in SNNs. Existing SNNs, such as \citep{hohn2001}, \citep{jin2018}, and \citep{zhang2019}, often follow the bio-plausible knowledge in neuroscience to set these internal hyper-parameters and then make a fine-tuning of all hyper-parameters. \citep{kulkarni2018} describe several hyper-parameter tuning experiments achieved on the MNIST data set. Carlson et al.~\cite{carlson2014} employ evolutionary algorithms to optimize the hyper-parameters within spiking neural models. However, there still lacks a theoretical analysis for guiding whether all hyper-parameters are important, which hyper-parameter is sensitive, and how to tune them. Our work, thanks to dynamical system theory, provides a theoretical framework for investigating spiking neural models. We point out the dynamical properties of SNNs with different neuron models and give effective suggestions on the setting of sensitive hyper-parameters.

Zero-order alternating optimization has been widely discussed in neural-network-related researches. These approaches can roughly be divided into two categories: (1) Optimizing the architectures of neural networks. For example, Lorenzo et al.~\citep{lorenzo2017} use particle swarm optimization to select the architectures of DNNs. (2) Optimizing internal hyper-parameters. Bergstra and Bengio~\cite{bergstra2012} provide a classical way that randomly samples hyper-parameters from a prior distribution and finds the best one. However, due to the limitation of observation information, the zero-order alternating optimization approaches intrinsically require large computation and storage.

\section{Spiking Neural Model} \label{sec:snm}
In this work, we focus more on the LIF model, not only due to the ease with which the LIF model can be analyzed and simulated but also because this model is the most widely used in the intersection of SNN and Artificial Intelligence. Here, we review a general form of the LIF model, which with $M$-dimensional input signals $\boldsymbol{I}(t) = \{I_1(t), \dots, I_M(t)\}$
\begin{equation} \label{eq:LIFsm}
	\tau_m \frac{d u}{d t} = - u + R~ f( \boldsymbol{I(t)} ),
\end{equation}
where $u(t)$ denotes the membrane potential of the concerned neuron at time $t$, $\tau_m$ and $R$ are positive-valued hyper-parameters with respect to membrane time and membrane resistance, respectively, and $f$ represents the aggregation function, usually expressed in a linear formation, $f( \boldsymbol{I(t)} ) = \sum_{j=1}^{M} \mathbf{W}_{j}I_j(t)$, where $\mathbf{W}_{j}$ denotes the learnable connection weight corresponding to the $j$-th input channel. Based on the \emph{spike response model} scheme~\citep{gerstner1995}, the LIF equation has a general solution with $u_{rest}=0$ as follows
\begin{equation} \label{eq:solution}
	u(t) = \sum_{j=1}^{M}\mathbf{W}_{j} \left[~ \int_{t'}^{t} \exp\left( \frac{t'-s}{\tau_m} \right) I_j(s) ~ds ~\right],
\end{equation}
where $t'$ denotes the last firing time $t' = \max \{ s ~|~ u(s) = u_{firing}, s < t \}$. An output stimulus $S(t)$ is generated whenever the membrane potential $u(t)$ reaches a certain threshold $u_{firing}$ (firing threshold). We formulate this procedure using a spike excitation function
\[
f_e: u \rightarrow S, ~~\text{where}~~ S(t) \triangleq \left\lfloor \frac{u(t)}{u_{firing}} \right\rfloor .
\]
After firing, the membrane potential is instantaneously reset to a lower value $u_{rest}$ (rest voltage). Some researchers take the delayed-firing behaviors in neuroscience, and thus, add an absolute refractory period $\Delta_{abs}$~\citep{hunsberger2015} or a refractory kernel $\eta_{ref}(\tau_m)$~\citep{dumont2017} to the LIF model. So Equation~\ref{eq:solution} becomes
\begin{equation} \label{eq:ref}
	u(t) = \sum_{j=1}^{M}\mathbf{W}_{j} \left[~ \int_{t'}^{t} \exp\left( \frac{t'-s}{\tau_m} \right) I_j(s) ~ds ~\right] + \eta_{ref}(\tau_m) * S(t),
\end{equation}
where $*$ denotes the convolution operation and $\eta_{ref}(\tau_m) * S(t)$ is the refractory response of a neuron. 

Notice that the LIF model with the initial condition $u(0) = u_{rest}$ or $u(t') = u_{rest}$ leads to an Ordinary Differential Equation (ODE) dynamical system. The eigenvalues $\rho$ of this dynamical system are equal to the quotient of $-1$ and $\tau_m$ by solving the following abstract equations
\[
\left\{
\begin{aligned}
	\frac{d u}{d t} &= - \frac{1}{\tau_m} u , \\
	u(0) &= u_{rest} .
\end{aligned}\right.
\]
The information about how to calculate the eigenvalues of ODEs is detailed in Subsection~\ref{subsec:eigenvalues}.

\begin{figure}[t]
	\centering
	\includegraphics[width=0.9\textwidth]{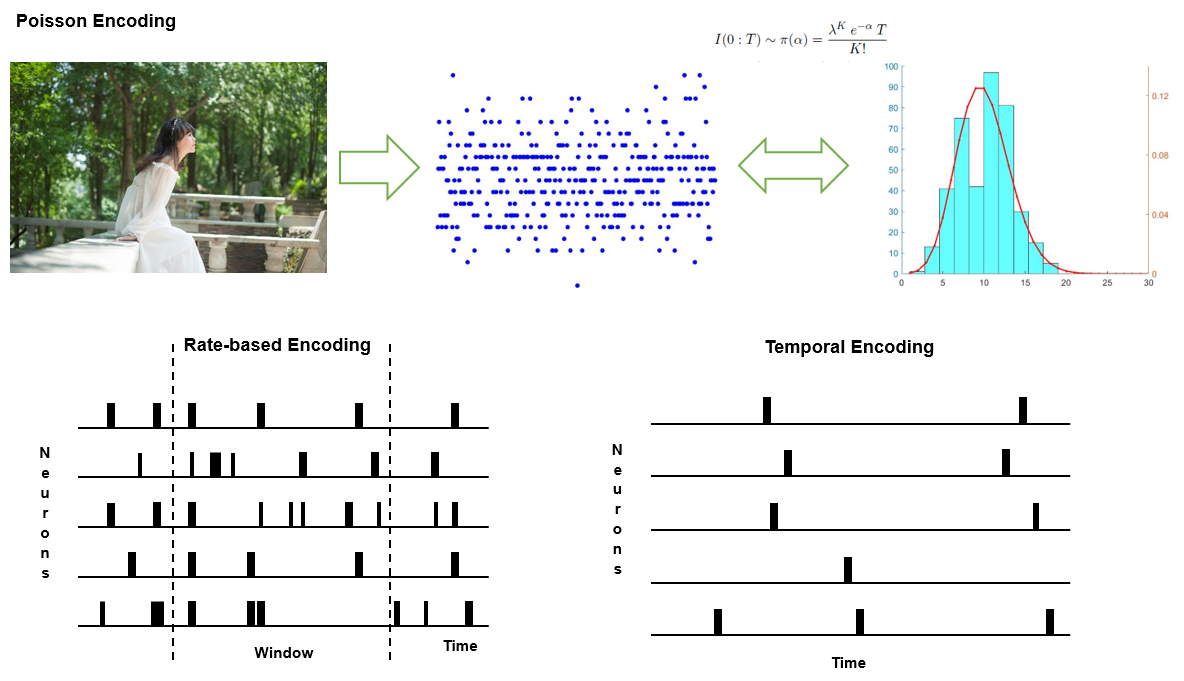}
	\label{fig:encoding}
	\caption{Illustrations of neural encoding in spiking versions.}
\end{figure} 
\subsection{Neural Encoding}
Before present to SNNs, the real input data should be pre-converted into a spiking version. There are two main categories of encoding approaches, \emph{rate-based encoding} and \emph{temporal encoding}, which are illustrated in Figure~\ref{fig:encoding}. In rate-based encoding~\citep{lobo2020}, the input data is encoded by the number of fired spikes within temporal windows. The representative approaches are encoded by a Poisson distribution and recorded by a Dynamic Vision Sensor~\citep{quiroga2005,anumula2018}, which is shown in Figure~\ref{fig:encoding}. In temporal encoding~\citep{yu2014}, the input data is encoded by the distance between time instances that fire spikes.

\subsection{Eigenvalues and Abstract Equations} \label{subsec:eigenvalues}
For a system of first-order linear differential equations
\[
\frac{d }{dt} \begin{pmatrix}
	u_1 \\
	\vdots \\
	u_n
\end{pmatrix} = \begin{pmatrix}
	\alpha_{11} &  \dots  &  \alpha_{1n} \\
	\vdots      &  \ddots  &  \vdots      \\
	\alpha_{n1} &  \dots  &  \alpha_{nn}
\end{pmatrix} \cdot \begin{pmatrix}
	u_1 \\
	\vdots \\
	u_n
\end{pmatrix} + \begin{pmatrix}
	f_1(t) \\
	\vdots \\
	f_n(t)
\end{pmatrix},
\]
we have its abstract formulation
\[
\frac{d }{dt} \begin{pmatrix}
	u_1 \\
	\vdots \\
	u_n
\end{pmatrix} = A \cdot \begin{pmatrix}
	u_1 \\
	\vdots \\
	u_n
\end{pmatrix} \quad\text{with}\quad A = \begin{pmatrix}
	\alpha_{11} &  \dots  &  \alpha_{1n} \\
	\vdots      &  \ddots  &  \vdots      \\
	\alpha_{n1} &  \dots  &  \alpha_{nn}
\end{pmatrix}  .
\]
These abstract equations are only related to the observation variables $u_1,\dots,u_n$. So the eigenvalues (spectrum values) of the matrix (operator) $A$ lead to the evolution rules of the dynamical system.

\section{Dynamical Properties of Spiking Neural Models} \label{sec:DS}
From Section~\ref{sec:snm}, the LIF model has two internal hyper-parameters, \emph{i.e.}, the membrane time $\tau_m$ and membrane resistance $R$. In this section, we will investigate the dynamical properties of spiking neural models concerning the hyper-parameters above. For convenience, we define a new hyper-parameter $\gamma = -1/\tau_m$, called \emph{control rate}.

\subsection{Bifurcation Dynamics of LIF-modeling SNNs} \label{subsec:LIF}
Our main conclusion is listed as follows.
\begin{thm} \label{thm:snn}
	Given the initial condition $u(0) = u_{rest}$ or $u(t') = u_{rest}$, the dynamical system led by one layer of LIF neurons is a bifurcation dynamical system, and control rate $\gamma$ is the corresponding bifurcation hyper-parameter.
\end{thm}
According to the bifurcation theory in dynamical systems, a bifurcation occurs when a small smooth change of the parameter values (often the bifurcation hyper-parameters pass through a critical point) causes a sudden topological change in its behavior~\citep{onuki2002,kuznetsov2013}. Theorem~\ref{thm:snn} shows that the bifurcation dynamical system led by one layer of LIF neurons is structurally unstable, and its performance is sensitive to the setting of the control rates.
\begin{proof}
	We start by formulating the total energy of the LIF model. From the perspective of ODE, the LIF model with initial condition $u(0) = u_{rest}$ or $u(t') = u_{rest}$ guides a continuous dynamical system. Multiply both sides of Equation~\ref{eq:LIFsm} by the membrane potential $u$, as commonly used in~\citep{hirsch2012}, we can obtain the energy of an one-LIF-neuron dynamical system
	\begin{equation}  \label{eq:energy}
		\mathcal{H}(u,t) = u^2 + \frac{2R}{\tau_m} F(u,t) = u^2 - 2\gamma R F(u,t),
	\end{equation}
	where $u^2$ indicates the kinetic energy and $F$ is the primitive function of $f \cdot u$, representing the non-potential forces. So $\mathcal{H}(u,t)$ is a Lyapunov-like function. Direct calculations show that:
	\begin{equation}  \label{eq:energy_time_der}
		\frac{d \mathcal{H}}{d t} = 2 u \frac{d u}{d t} + \frac{2R}{\tau_m} f u  
		= 2u \left( \frac{d u}{d t} + \frac{2R}{\tau_m} f \right)
		= -\frac{2}{\tau_m} u^2 = 2 \gamma u^2.
	\end{equation}
	Based on Equations~\ref{eq:energy} and~\ref{eq:energy_time_der}, both the energy function $\mathcal{H}(u,\gamma,t)$ and its time derivative are dominated by the control rate. If the control rate is set as a negative constant as usual, we have $d \mathcal{H} / d t < 0$. Thus, the dynamical system led by one LIF neuron is energy-dissipating. This means that information will be continuously lost during the learning process of one LIF neuron and the concerned spiking neuron appears to hinder spike excitation. When $\gamma = 0$, one LIF neuron becomes a conservative system where the conversation from received signals to biological spikes is without damage, \emph{i.e.}, the Integrate-and-Fire (IF) model. When $\gamma > 0$, the system energy would be increasing since $d \mathcal{H} / d t > 0$, and thus, the concerned neuron is encouraged to firing spikes.
	
	Next, we extend the aforementioned results relative to one LIF neuron to a fully-connected feed-forward SNN. Similarly, by adding apposite initial conditions to the concerned spiking equations, the total energy of a layer of $N$ LIF neurons can be given by
	\begin{equation}  \label{eq:energy_N}
		\mathcal{H}_N(\boldsymbol{u},t) = |\boldsymbol{u}|^2 + \frac{2R}{\tau_m} F(\boldsymbol{u},t),
	\end{equation}
	where $\boldsymbol{u}$ is a $N$-dimensional vector that the $i$-th component represents the membrane potential of the $i$-th spiking neuron, and $F(\boldsymbol{u},t)$ is the primitive function of $\langle \boldsymbol{f}, \boldsymbol{u} \rangle$. The time derivative of the total energy function can be calculated by
	\begin{equation}  \label{eq:energy_time_der_N}
		\frac{d \mathcal{H}_N}{d t} = 2 \gamma |\boldsymbol{u}|^2.
	\end{equation}
	Similar to Equation~\ref{eq:energy_time_der}, the control rate is proportional to the eigenvalue $\rho$ of LIF equation, \emph{i.e.}, $\gamma = \rho$. So the time derivative of energy function, described by Equation \ref{eq:energy_time_der}, becomes
	\[
	\frac{d \mathcal{H}_N}{d t} = 2\rho |\boldsymbol{u}|^2.
	\] 
	A bifurcation occurs once the control rate $\gamma$, as well as the eigenvalue $\rho$ of this dynamical system, crosses the critical value 0; $\rho > 0$ leads to an unstable manifold, while $\rho < 0$ leads to a stable one. Therefore, $\mathcal{H}_N(\boldsymbol{u},t)$, described by Equation~\ref{eq:energy_N}, coincides with a bifurcation dynamical system concerning control rate $\gamma$. This completes the proof.
\end{proof}

Based on the results above, the control rate $\gamma$ relative to the membrane time hyper-parameter $\tau_m$ plays a crucial role on SNNs, directly affecting the topology of the dynamical system. As conventionally preset $\gamma$ to a negative value in the whole network, SNNs comprise the stacking of dissipating systems. Thus, the energy will be continuously lost layer by layer during the learning process of SNNs. A better way is to adaptively learn the control rate so that SNNs can determine the desired system mode (\emph{i.e.}, dissipating, conservative, diffusion systems, or a mixture of the three) according to the actual environment. On the other hand, existing SNNs usually employ unified control rates, that is, $\gamma_1 = \dots = \gamma_N$. According to Theorem~\ref{thm:snn} and Equation~\ref{eq:energy_time_der_N}, this manner makes the topology structure of whole SNNs tied to only one pre-given hyper-parameter, thus greatly weakening the representation ability of SNNs. To revise this manner, we force the control rates in one layer to be diverse, \emph{i.e.}, each neuron has an exclusive control rate parameter. In summary, the traditional manners that employ unified control rates and preset all control rates to one fixed and negative constant impede the flexibility of SNNs. To achieve adaptive systems, the control rates in SNNs should be diverse and learnable.

The change of the membrane resistance hyper-parameter $R$ and connection weights $\mathbf{W}$ only affect the system energy (see Equations~\ref{eq:energy} and \ref{eq:energy_N}), but has nothing to do with their intrinsic dynamical properties (according to Equations~\ref{eq:energy_time_der} and \ref{eq:energy_time_der_N}). Therefore, the role of control rates cannot be replaced by learnable connection weights and other hyper-parameters.

\subsection{Approaches for Parameterizing Control Rates} \label{subsec:parameterizing}
An intuitive idea for achieving adaptive SNN systems is to parameterize the control rates. However, training SNNs with parametric control rates is a brand-new challenge since there is almost no experience of training hyper-parameters in neural networks to borrow. The difficulty is twofold: (1) The existing SNNs are almost trained based on the spike response model scheme. This leads to the membrane potential in Equation~\ref{eq:solution} is dominated by an indirect product interaction of connection weights $\mathbf{W}_j$ and control rate $\gamma$. So we cannot optimize the control rates and connection weights in parallel. (2) The roles of control rates and connection weights are distinctly different; each control rate is convolved with the received spikes aggregated by connection weights. So the spike errors caused by control rates spread temporally, while connection weights only transmit errors between layers. Formally, let $E$ be the loss, then we have
\[
\frac{\partial E}{\partial \gamma_k} = \frac{\partial E}{\partial u_k} \frac{\partial u_k}{\partial \gamma_k} = \frac{\partial E}{\partial u_k} \left( \sum_{j=1}^{M}\mathbf{W}_{j} \left[~ \int_{t'}^{t} \exp\left( \gamma_k(s-t') \right) (s-t') I_j(s) ~ds ~\right] \right),
\]
where the subscript $k$ denotes the $k$-th spiking neuron. When the firing period of the concerned spiking neuron is too small, \emph{i.e.}, $(t-t') \to 0$, the gradients appear to vanish as $\partial u_k / \partial \gamma_k \to 0$; when the firing period is slightly longer, \emph{i.e.}, $t \gg t'$, the gradients will be of explosion though time. To sum up, simply utilizing gradient-based methods to train the control rates is easier to fall into the quagmire of gradient vanishing and explosion.

An alternative approach to alleviate the issue is to employ alternating optimization for estimating control rate hyper-parameters. The key idea is to regard the control rates as a group of hyper-parameters drawn from a prior distribution so that solving for each learnable variable (connection weights or control rates) reduces to well-known methods. More formally, we consider a fully-connected feed-forward architecture and list this optimization procedure as follows:
\begin{itemize}
	\item[\textbf{Initialization:}] Sampling a group of control rates $\boldsymbol{\gamma}=\{\gamma_k\}$ from a pre-given distribution, such as the uniform distribution $\mathcal{U}[-1,1]$. Then spikes spread according to
	\[
	\left\{\begin{aligned}
		S_j^{1}(t) &= I_j(t), \\
		u_k^{l+1}(t) &= \sum_{j=1}^{n_l}\mathbf{W}_{kj}^l \left[ \int_{t'}^{t} 
		\exp\left( \gamma_k(s-t') \right)S_j^{l}(s) ds \right], \\
		S_k^{l+1}(t) &\leftarrow f_e\left( u_k^{l+1}(t) \right),
	\end{aligned}\right.
	\]
	where $n_l$ is the number of spiking neurons in the $l$-th layer and $\mathbf{W}$ is the connection weight matrix. The superscript $l$ denotes the $l$-th layer and the subscripts $k$ and $j$ indicate the $k$-th and $j$-th spiking neurons, respectively. 
	\item[\textbf{Update connection weights:}] How to update the connection weights $\boldsymbol{W}$ with fixed $\boldsymbol{\gamma}$ depends on the choice of error-propagation techniques. Here, we employ a seminal work, SLAYER~\citep{shrestha2018} as a basic model.
	\item[\textbf{Update $\boldsymbol{\gamma}$:}] We solve $\arg\min\nolimits_{\boldsymbol{\gamma}} Loss(\boldsymbol{\gamma} | \boldsymbol{I}, \boldsymbol{Y}; \boldsymbol{W}) $, where $\boldsymbol{Y}$ denotes the supervised signals. Thus, existing algorithms, such as alternating coordinate descent can be applied directly to find a collection of apposite control rates.
\end{itemize}
The approaches based on alternating optimization place larger demands on computation and storage, and usually converge slowly in neural networks.

\subsection{Dynamical Properties for Other Spiking Neural Models} \label{subsec:others}
Adhering to the framework mentioned in Subsection~\ref{subsec:LIF}, we here introduce several famous ones and analyze their dynamical properties. The first concerned model is the HH model~\citep{hodgkin1952}, which usually be described as follows
\[
\left\{\begin{aligned}
	\tau_u \frac{ ~u}{d t} &= -g_1 u - g_2 um^3h - g_3 un^4 + f(\boldsymbol{I}(t)), \\
	\tau_n \frac{d n}{d t} &= -n, \\
	\tau_m \frac{d m}{d t} &= -m, \\
	\tau_h \frac{d h}{d t} &= -h, \\
\end{aligned}\right.
\]
where $u$ is the concerned electro-chemical membrane variable, $n$, $m$, and $h$ describe the opening and closing of the voltage-dependent channels, $\tau_u$ is the capacitance of the membrane $u$, $\tau_n$, $\tau_m$, and $\tau_h$ are the corresponding membrane time hyper-parameters, $-g_1$, $g_2$, and $g_3$ denote the conductance parameters for the different ion channels (\emph{e.g.}, sodium $\mathrm{Na}$ and potassium $\mathrm{K}$), and $f$ represents the aggregation function of input signals $\boldsymbol{I}(t)$, see Section~\ref{sec:snm}.
\begin{corollary}
	The HH model with positive hyper-parameters (\emph{i.e.}, $g_1$, $\tau_u$, $\tau_n$, $\tau_m$, and $\tau_h$) is a dissipative system.
\end{corollary}
\begin{proof}
	The energy function of the HH model can be given by $\mathcal{H}_{HH}(u,n,m,h,t) = |\boldsymbol{u}|^2 + o(|\boldsymbol{u}|) + 2 F(u,t)$, where $\boldsymbol{u} = (u,n,m,h)$ is a short notation and $o(|\boldsymbol{u}|) = -2 ( g_2m^3h + g_3n^4h ) u^2$ denotes the high-order term of $|\boldsymbol{u}|$, and $F$ is the primitive function of $f \cdot u$, representing the non-potential forces. Thus, we can obtain its time derivative
	\[
	\frac{d \mathcal{H}_{HH}}{d t} = - 2 (g_1\frac{u^2}{\tau_u} + \frac{n^2}{\tau_n} + \frac{m^2}{\tau_m} + \frac{h^2}{\tau_h} ).
	\]
	Therefore, we have $d \mathcal{H}_{HH} / d t <0$ since $g_1, \tau_u, \tau_n, \tau_m, \tau_h >0$ This completes the proof.
\end{proof}

We also investigate Izhikevich's neuron model~\citep{izhikevich2003}, which is a good compromise between biophysical plausibility and computational cost. It is usually simple formulated by the following coupled equations
\[
\left\{\begin{aligned}
	\frac{d u}{d t} &= a_u u^2 + b_u u -w + f(\boldsymbol{I}(t)), \\
	\frac{d w}{d t} &= a_w( b_w u - w ),
\end{aligned}\right.
\]
where $u$ and $w$ are dimensionless membrane variables, $a_u$, $b_u$, $a_w$, and $b_w$ are (positive) dimensionless hyper-parameters, and $f$ represents the corresponding aggregation function of input signals $\boldsymbol{I}(t)$. Izhikevich's neuron model usually employs a group of typical values, \emph{i.e.}, $a_u=0.04$, $b_u=5$, $a_w=0.02$, and $b_w=0.2$.
\begin{corollary}
	Izhikevich's neuron model leads to a hyperbolic system.
\end{corollary}
\begin{proof}
	Write its energy function $\mathcal{H}_{I}(u,w,t) = |(u,w)|^2 + o(|u|) + 2 F(u,t)$, where $o(|u|) = 2 a_u u^3$ denotes the high-order term of $|u|$ and $F$ is the primitive function of $f \cdot u$, representing the non-potential forces. And then, the time derivative is at your fingertips
	\[
	\frac{d\mathcal{H}_{I}}{dt} = 2 (u,w) ~A~ (u,w)^T \quad \text{with} \quad A = \begin{pmatrix}
		b_u &  0 \\
		a_w b_w & -a_w \\
	\end{pmatrix}  .
	\]
	Izhikevich's neuron model has more complicated dynamical properties. Since the matrix $A$ has two sign-opposite eigenvalues, \emph{i.e.}, $\rho_1 = b_u >0$ and $\rho_2 = -a_w <0$, the tangent space becomes a hyperbolic manifold. So the neuron has two fixed points: a saddle point and an attractor. Around the saddle point, the dynamical system is unstable; the membrane potential will not stay at this point and the neuron either tends to fire or returns to near $0$. However, the dynamical system around the attractor is relatively stable. Once the membrane potential trap in this region, it is difficult for neurons to fire again.
\end{proof}

\section{Bifurcation Spiking Neural Networks} \label{sec:BSNN}
We present the BSNN for achieving an SNN with adaptive dynamics. In contrast to the LIF model that the eigenvalue is proportional to its control rate, BSNN employs a group of trainable parameters to separate the eigenvalues of the spiking neuron model from the control rates, making it possible to achieve diverse and learnable eigenvalues. Formally, we present the bifurcation spiking neurons model as follows
\begin{equation}  \label{eq:preBSNN}
	\frac{\partial \boldsymbol{u} (t)}{\partial t} = \gamma\boldsymbol{u}(t) + \boldsymbol{\lambda} \boldsymbol{u}^* + \frac{R}{\tau_m} \boldsymbol{f}(\boldsymbol{I}),
\end{equation}
where $\gamma$ is the control rate and and $\boldsymbol{\lambda}$ is a bifurcation parameter matrix. The vector $\boldsymbol{u}^* = (u_1^*, \dots, u_N^*)$ portrays the mutual promotion between neurons. Here, we unfold the $k$-th variable as $u_k^* = \sum\nolimits_{i \neq k} u_i + o(|u_k|)$, where $o(|u_k|)$ denotes the high-order term of $u_k$. Thus, Equation~\ref{eq:preBSNN} becomes
\[
\left\{\begin{aligned}
	\frac{\partial u_1(t)}{\partial t} &= \gamma~ u_1(t) + \sum_{i\neq 1} \lambda_{1i} u_i + o(|u_1|) + \frac{R}{\tau_m} f_1(\boldsymbol{I}) ,\\
	\vdots & \quad\quad\quad\quad\quad\quad\quad\quad\quad \vdots\\
	\frac{\partial u_N(t)}{\partial t} &= \gamma~ u_N(t) + \sum_{i\neq N} \lambda_{Ni} u_i + o(|u_N|) + \frac{R}{\tau_m} f_N(\boldsymbol{I})  ,
\end{aligned}\right. 
\]
where
\[
f_k(\boldsymbol{I}) = \sum_{j=1}^{M} \mathbf{W}_{kj}I_j(t).
\]
The basic building block of BSNN is a system of equations concerning a cluster of spiking neurons. Invoke this cluster of spiking neurons to constitute a layer and reuse Equation~\ref{eq:preBSNN} layer by layer. Then we can establish a feed-forward multi-layer network.

\subsection{Adaptive Dynamical System}
To ensure BSNN becomes an adaptive dynamical system, we need to verify that the time derivative of the energy function produced by BSNN has learnable and diverse eigenvalues. Similar to the analysis in Section \ref{sec:DS}, we have the total energy of the BSNN model as follows
\[
\mathcal{H}_B(\boldsymbol{u},t) = |\boldsymbol{u}|^2 + \frac{2R}{\tau_m} F(\boldsymbol{u},t).
\]
Correspondingly, the time derivative of energy function can be calculated by
\[
\frac{d \mathcal{H}_B}{d t} = 2 \boldsymbol{u}^T L_{\boldsymbol{\lambda}} \boldsymbol{u} \quad\text{with}
\quad  L_{\boldsymbol{\lambda}} = \begin{pmatrix}
	\gamma &  & \\
	& \ddots & \\
	& & \gamma
\end{pmatrix} + \boldsymbol{\lambda}.
\]
Further, we can declare that BSNN has non-trivial solutions for achieving adaptive dynamical system.
\begin{thm}  \label{thm:BSNN}
	If the bifurcation hyper-parameters $\lambda_{ij}$ are all great than 0, there are at most $2^{N-1}$ bifurcation solutions in Equation~\ref{eq:preBSNN}.
\end{thm}
\begin{proof}
	Theorem~\ref{thm:BSNN} can be roughly proved by the following steps. First, finding the characteristic roots of our proposed BSNN model. According to Equation~\ref{eq:preBSNN}, we can obtain its abstract representation
	\[
	\frac{d\boldsymbol{u}}{dt} = L_{\boldsymbol{\lambda}}\boldsymbol{u} + G(\boldsymbol{u},\boldsymbol{\lambda}) \quad\text{with}\quad L_{\boldsymbol{\lambda}} = A + B_{\boldsymbol{\lambda}} \quad\text{and}\quad G(\boldsymbol{u},\boldsymbol{\lambda})=o(|\boldsymbol{u}|) ,
	\]
	where
	\[
	A = \begin{pmatrix}
		\gamma &  & \\
		& \ddots & \\
		&   & \gamma
	\end{pmatrix} \quad\text{and}\quad B_{\boldsymbol{\lambda}} = \begin{pmatrix}
		0      & \lambda_{12}    & \dots   & \lambda_{1n} \\
		\lambda_{21} &    0      & \dots   & \lambda_{2n} \\
		\vdots    & \vdots       & \ddots   & \vdots       \\
		\lambda_{n1} & \lambda_{n(n-1)}     & \dots    & 0
	\end{pmatrix}.
	\]
	Suppose that the eigenvalues of the matrix $B_{\boldsymbol{\lambda}}$ are $\beta_1,\dots,\beta_n$. So the eigenvalue $\rho_i$ of $L_{\boldsymbol{\lambda}}$ can be calculated as the sum of that of $A$ and that of $B_{\boldsymbol{\lambda}}$, that is,
	\[
	\rho_i = -\gamma + \beta_i.
	\]
	Next, we can elucidate the bifurcation solutions relative to the eigenvalues. For simplicity, we take the 2-neuron model as an example, that is, 
	\[
	A = \begin{pmatrix}
		\gamma & 0 \\
		0   & \gamma
	\end{pmatrix} \quad\text{and}\quad B_{\boldsymbol{\lambda}} = \begin{pmatrix}
		0   & \lambda_1\\
		\lambda_2 & 0
	\end{pmatrix}.
	\]
	Let $D_1 = 2\gamma$ and $D_2 = \gamma^2 - \lambda_1\lambda_2$. when $\Delta = D_1^2 - 4 D_2 = \lambda_1\lambda_2 \geq 0$, $L_{\boldsymbol{\lambda}}$ has two real eigenvalues
	\[
	\rho_1 = \frac{-D_1-\sqrt{D_1^2-4D_2}}{2} \quad\text{and}\quad  \rho_2 = \frac{-D_1+\sqrt{D_1^2-4D_2}}{2}.
	\]
	Obviously, $\rho_1$ must be less than zero, whereas $\rho_2$ is indefinite. Let $\lambda_c = \gamma^2$ be the critical threshold, then the bifurcation solutions of Equation~\ref{eq:preBSNN} are dominated by one pair of bifurcation eigenvalues
	\[
	\left\{ \begin{aligned}
		\rho_1 &= \frac{-D_1-\sqrt{D_1^2-4D_2}}{2} < 0, \\
		\rho_2 &= \frac{-D_1+\sqrt{D_1^2-4D_2}}{2}\begin{cases}
			<0, &\lambda_1\lambda_2 < \lambda_c;\\
			=0, &\lambda_1\lambda_2 = \lambda_c;\\
			>0, &\lambda_1\lambda_2 > \lambda_c.
		\end{cases}
	\end{aligned} \right.
	\]
	As long as the product of $\lambda_{1}$ and $\lambda_2$ is greater than 0, there exists at least one non-trivial solution of Equation~\ref{eq:preBSNN}. In detail, when $\lambda_1\lambda_2 < \lambda_c$, both eigenvalues are negative, and thus, the whole dynamical system consists of two dissipative sub-systems, where both spiking neurons tend to hinder spike excitation. When $\lambda_1\lambda_2 = \lambda_c$, the whole dynamical system consists of a dissipative sub-system relative to the negative eigenvalue $\rho_1$ and a conservative one that $\rho_2$ is equal to 0. When $\lambda_1\lambda_2 > \lambda_c$, a new bifurcation phenomenon occurs, the whole dynamical system becomes structurally unstable, intuitively, one neuron still works in a ``leaky" mode, but the other one contributes to spike excitation.
	
	The existence of bifurcation eigenvalues is equivalent to the existence of non-trivial solutions of Equation~\ref{eq:preBSNN}, and one pair of bifurcation solutions induces a group of apposite eigenvalues for achieving adaptive dynamical systems. Generally, for the case of $N$ neurons, the solution of Equation~\ref{eq:preBSNN} possesses at most $2^{N-1}$ bifurcation solutions.
\end{proof}

Based on the results of Theorem~\ref{thm:BSNN}, the eigenvalues of Equation~\ref{eq:preBSNN} are dominated by a series of bifurcation parameters $\boldsymbol{\lambda}$. So we can convert the issue of searching for apposite control rate hyper-parameters into a new problem of how to train the bifurcation parameters. Therefore, even if the control rate is fixed, BSNN can still achieve an adaptive dynamic system. The training procedure for BSNN is introduced in the next subsection.

\subsection{Implementation}
Consider a feed-forward BSNN with $M$ pre-synaptic input channels and $N$-dimensional spiking neurons, and approximate the mutual promotion from the $i$-th neuron to the $k$-th neuron is caused by the last spike of neuron $i$, noted as $S_i(t'_i)$, where $t'_i = \max \{ s | u_i(s) = u_{firing}, s < t_i \}$. For the $k$-th neuron, we have
\begin{equation} \label{eq:feedforward}
	\frac{d u_k(t)}{d t} = \gamma u_k(t) + \sum_{i=1, i\neq k}^{N} \lambda_{ki} S_i(t'_i) +  \sum_{j=1}^{M} \mathbf{W}_{kj} I_j(t).
\end{equation}
According to Equation~\ref{eq:feedforward}, BSNN has two types of learnable parameters, \emph{i.e.}, bifurcation parameters $\boldsymbol{\lambda}$ and connection weights $\mathbf{W}$, where $\boldsymbol{\lambda}$ is linearly independent to $\mathbf{W}$ at time $t$. Thus, BSNN avoids the problem of parameters entanglement.

Akin to the spike response model scheme~\citep{gerstner1995}, Equation~\ref{eq:feedforward} has a closed-form solution
\begin{equation} \label{eq:synaptic term}
	u_k(t) = \int_{t'}^{t} \exp\left( \gamma(s-t') \right) \cdot Q_k(s) ds \quad\text{with}\quad Q_k(t) = \sum_{j=1}^{M} \mathbf{W}_{kj} I_j(t) + \sum_{i\neq k} \lambda_{ki} S_i(t'_i).
\end{equation}
Finally, the generated spike is transmitted to the next neuron via the spike excitation function $f_e$.

\subsection{Error Backpropagation in BSNN}
BSNN with supervised signals can also be optimized via error backpropagation. Firstly, we denote the input (spike) sequence to a neuron as the following general form~\citep{huh2018}
\[
I_j(t) = \sum_{firing} \epsilon_j\left( t-t_j^{firing} \right),
\]
where $t_j^{firing}$ is the spike time of the $j$-th input and $\epsilon(t)$ is a corresponding Dirac-delta function. Summing up the loss of the $k$-th target supervised signal $\hat{S_k}(t)$ related to $S_k(t)$ in time interval $[0,T]$: 
\begin{equation} \label{eq:cost}
	E_k = \frac{1}{2} \int_{0}^{T} E_k(t) dt 
	= \frac{1}{2} \int_{0}^{T}
	\left(S_k(t) - \hat{S_k}(t) \right)^2 dt .
\end{equation}
So for time $t$, we have
\begin{equation} \label{eq:1}
	\frac{\partial E_k(t)}{\partial \mathbf{W}_{kj}} 
	= \frac{\partial E_k(t)}{\partial S_k} 
	\frac{\partial S_k}{\partial u_k} 
	\frac{\partial u_k}{\partial \mathbf{W}_{kj}} ,
\end{equation}
where the first term is the error backpropagation of the excitatory neurons, the second term is that of the generated spikes with respect to the membrane potential, and the third term denotes the that of basic bifurcation neuron. Plugging Equation~\ref{eq:synaptic term} and Equation~\ref{eq:cost} into Equation~\ref{eq:1}, we have
\[
\frac{\partial E_k(t)}{\partial \mathbf{W}_{kj}} = \left(S_k(t) - \hat{S_k}(t) \right)
f_e'(u_k)
\delta_{j}(t) \quad\text{with}\quad \delta_{j}(t) = \frac{\epsilon_j(t)}{\tau_m}
\exp\left( \gamma t \right) .
\]
However, the derivative of the spike excitation function $f_e'(u)$ is a persistent problem for training SNNs with supervised signals. Recently, there have emerged many seminal approaches for addressing this problem. In this paper, we directly employ the result of~\citep{shrestha2018}. Therefore, we obtain the backpropagation pipeline related to connection weights $\mathbf{W}_{kj}$
\[
\nabla_{\mathbf{W}_{kj}} E = \int_{0}^{T} \frac{\partial E_k(t)}{\partial \mathbf{W}_{kj}} dt.
\]
Similarly, the correction formula with respect to $\lambda_{ki}$ is given by
\[
\nabla_{\lambda_{ki}} E = 
\int_{0}^{T} \left( \hat{S_k}(t)-S_k(t) \right)
f_e'(u_k)
S_i(t'_i) \gamma
\exp\left( \gamma t \right) dt.
\]
We can also add a learning rate $\eta$ to help convergence, just like most deep artificial neural networks.

Here, BSNN is implemented by an extended backpropagation algorithm. Compared with the existing SNNs, BSNN only needs to calculate one more set of gradients, \emph{i.e.}, $\nabla_{\lambda_{ki}} E$ during feedback. The records of $S_i(t'_i)$ do not cause additional storage because we intrinsically need the membrane potential values of each spiking neuron during the gradient calculation procedure as shown in Equation~\ref{eq:1}. So both the computation and storage of BSNN are considerably less in comparison with the alternating optimization approaches. \\

In summary, the proposed BSNN has at least three advantages: (1) BSNN employs learnable parameters to adjust eigenvalues, and thus, its topology becomes flexible, enabling SNNs to capture the adaptive learning behavior of spiking neurons to the environment change. (2) Since the new-added bifurcation parameters $\boldsymbol{\lambda}$ are independent of the connection weights $\mathbf{W}$, the two will not be entangled and conflict when using the spike response model scheme to train SNNs. (3) The performance of BSNN depends less on the control rates.

\section{Experiments} \label{sec:Experiments}
In this section, we conducted experiments on four benchmark data sets to evaluate the functional performance of BSNN. The experiments are performed to discuss the following questions:
\begin{itemize}
	\item[Q1:] Is the performance of BSNN comparable with state-of-the-art SNNs?
	\item[Q2:] Does the performance of BSNN surpass that of alternating optimization, especially in terms of accuracy and efficiency?
	\item[Q3:] Concerning BSNN, is the performance robust to the control rate? In which conditions?
\end{itemize}

\textbf{Data Sets}: (1) The MNIST handwritten digit data set\footnote{\url{http://yann.lecun.com/exdb/mnist/}} comprises a training set of 60,000 examples and a testing set of 10,000 examples in 10 classes, and each example is centered in a $28 \times 28$ image. Using Poisson encoding, we produce a list of spike signals with a formation of $784 \times T$ binary matrices, where $T$ denotes the encoding length and each row represents a spike sequence at each pixel. (2) The Neuromorphic-MNIST (N-MNIST) data set\footnote{\url{https://www.garrickorchard.com/datasets/n-mnist}}~\citep{orchard2015} is a spiking version of the original frame-based MNIST data set. Each example in N-MNIST was converted into a spike sequence by mounting the ATIS sensor on a motorized pan-tilt unit and having the sensor move while it views MNIST examples on an LCD monitor. It consists of the same 60,000 training and 10,000 testing samples as the original MNIST data set, and is captured at the same visual scale as the original MNIST data set ($28 \times 28$ pixels) with both ``on" and ``off" spikes. (3) The Fashion-MNIST data set\footnote{\url{https://www.kaggle.com/zalando-research/fashionmnist}} consists of a training set of 60,000 examples and a testing set of 10,000 examples. And each example is a 28x28 grayscale image, associated with a label from 10 classes. (4) The Extended MNIST-Balanced (EMNIST)~\citep{cohen2017} data set is an extension of MNIST to handwritten, which contains handwritten upper \& lower case letters of the English alphabet in addition
to the digits, and comprises 112,800 training and 18,800 testing samples for 47 classes.

\begin{table}[!htb]
	\centering
	\caption{Parameter Setting of BSNN on Various Data Sets.}
	\label{tab:paras}
	\resizebox{0.8\textwidth}{!}{%
		\begin{tabular}{c | c c c c c}
			\toprule
			\textbf{Parameters} {\textbf{Value}} & \textbf{MNIST} & \textbf{N-MNIST} & \textbf{Fashion-MNIST} & \textbf{EMNIST} \\
			\midrule
			Batch Size                & 32    & 32    & 32    & 64   \\
			Encoding Length $T$       & 300   & 300   & 400   & 400  \\
			Expect Spike Count(True)  & 100   & 80    & 100   & 140  \\
			Expect Spike Count(False) & 10    & 5     & 10    & 0    \\
			Learning Rate $\eta$      & 0.01  & 0.01  & 0.001 & 0.01 \\
			Refractory Period         & 2 ms  & 1 ms  & 2 ms  & 2 ms  \\
			Time Constant of Synapse $\tau_s$ & 8 ms  & 8 ms  & 8 ms  & 8 ms  \\
			\bottomrule
	\end{tabular} }
\end{table}
\textbf{Customization}: The pre-processing steps for these experiments are the same as~\citep{pillow2005,quiroga2005,anumula2018}, that is, each static image of (1) MNIST, (3) Fashion-MNIST, and (4) EMNIST is transformed as a spike sequence using Poisson Encoding, while each instance in N-MNIST was encoded by a Dynamic Audio / Vision Sensor (DAS / DVS). Note that the input data fed up to SNNs needs to go through some special processing, \emph{i.e.}, neural encoding that converts the static and non-spiking image to spikes. For these image classification tasks, we set 10 output spiking neurons corresponding to the classification labels. The output label of SNNs is the one with the greatest spike count. Table~\ref{tab:paras} lists the typical constant values in BSNNs. 

\textbf{Contenders}: We also employ two types of contenders to competing with the proposed BSNN: (1) several state-of-the-art SNNs with the spike response model scheme and (2) alternating optimization algorithms, see Subsection~\ref{subsec:parameterizing}. In this work, all SNN models are without any convolution structure. And the alternating optimization algorithms pre-sample a group of control rates from two uniform distributions, $U_1 = \mathcal{U}[-1,0]$ and $U_2 = \mathcal{U}[-1,1]$. For example, SLAYER-$U_1$ denotes an alternating optimization method based on SLAYER, which draws control rate hyper-parameters from $\mathcal{U}[-1,0]$.

\textbf{Experimental Results}: Table~\ref{tab:benchmark} lists the comparative performance (accuracy) and configurations (setting and epoches) of the contenders and BSNN on 3 digit data sets. As we can see, BSNN performs best against other competing approaches, achieving very superior testing accuracy (\emph{i.e.}, more than 99\% on MNIST, around 99.24\% on NMNIST, and more than 91\% on Fashion-MNIST). It is a laudable result for SNNs. In addition, the alternating optimization algorithms, \emph{i.e.}, both SLAYER-$U_1$ and SLAYER-$U_2$ steadily surpass the original SLAYER algorithm, which demonstrates the way of achieving diverse and adaptive control rates is significant and effective for improving the performance of SNNs.
\begin{table*}[t]
	\centering
	\caption{The comparative performance of the contenders and BSNN.}
	\label{tab:benchmark}
	\resizebox{1\textwidth}{!}{%
		\begin{threeparttable}
			\begin{tabular}{@{}cccccc@{}}
				\toprule
				\textbf{Data Sets}  &  \textbf{Contenders}  &  \textbf{Accuracy} (\%)
				& \textbf{Setting}  &  \textbf{Control Rate} ($\gamma$) & \textbf{Epochs} \\
				\midrule
				\multirow{9}{*}{MNIST} 
				& Deep SNN~\citep{o2016}
				& 97.80 & 28$\times$28-300-300-10 $\spadesuit$ & - & 50 \\
				& Deep SNN-BP~\citep{lee2016} & 98.71 & 28$\times$28-800-10 & - & 200 \\
				& SNN-EP $\heartsuit$         & 97.63 & 28$\times$28-500-10 & - & 25 \\
				& HM2-BP~\citep{jin2018}
				& 98.84 $\pm$ 0.02 & 28$\times$28-800-10     & - & 100 \\
				& SNN-L~\citep{lotfi2020}  & 98.23 $\pm$ 0.07 & 28$\times$28-1000-R28-10 & - & - \\
				& SLAYER~\citep{shrestha2018}
				& 98.39 $\pm$ 0.04 & 28$\times$28-500-500-10 & - & 50  \\
				& SLAYER-$U_1$ $\clubsuit$
				& 98.53 $\pm$ 0.03 & 28$\times$28-500-500-10 & - & -   \\
				& SLAYER-$U_2$ & 98.59 $\pm$ 0.01 & 28$\times$28-500-500-10 & - & - \\
				& BSNN (this work)
				& \textbf{99.02 $\pm$ 0.04} & 28$\times$28-500-500-10 & -0.21 & 50 \\
				\midrule
				\multirow{7}{*}{N-MNIST}
				& SKIM~\citep{cohen2016}  & 92.87 & 2*28$\times$28-10000-10 & - & -   \\
				& Deep SNN-BP             & 98.78 & 2*28$\times$28-800-10   & - & 200 \\
				& HM2-BP       & 98.84 $\pm$ 0.02 & 2*28$\times$28-800-10   & - & 60  \\
				& SLAYER       & 98.89 $\pm$ 0.06 & 2*28$\times$28-500-500-10 & - & 50 \\
				& SLAYER-$U_1$ & 99.01 $\pm$ 0.01 & 2*28$\times$28-500-500-10 & - & -  \\
				& SLAYER-$U_2$ & 99.07 $\pm$ 0.02 & 2*28$\times$28-500-500-10 & - & -  \\
				& BSNN (this work) 
				& \textbf{99.24 $\pm$ 0.12} & 2*28$\times$28-500-500-10 & -0.49 & 50 \\
				\midrule
				\multirow{6}{*}{Fashion-MNIST}
				& HM2-BP     & 88.99             & 28$\times$28-400-400-10 & - & 15   \\
				& SLAYER     & 88.61 $\pm$ 0.17  & 28$\times$28-500-500-10 & - & 50   \\
				& SLAYER-$U_1$  & 90.53 $\pm$ 0.04  & 28$\times$28-500-500-10 & - & - \\
				& SLAYER-$U_2$  & 90.61 $\pm$ 0.02  & 28$\times$28-500-500-10 & - & - \\
				& ST-RSBP~\citep{zhang2019}
				& 90.00 $\pm$ 0.13 & 28$\times$28-400-R400-10 $\diamondsuit$ & - & 30 \\
				& BSNN (this work)
				& \textbf{91.22 $\pm$ 0.06} & 28$\times$28-500-500-10 & -0.32 & 50   \\
				\midrule
				\multirow{6}{*}{EMNIST}
				& eRBP~\citep{neftci2017}
				& 78.17             & 28$\times$28-200-200-47 & - & 30   \\
				& HM2-BP     & 84.43 $\pm$ 0.10  & 28$\times$28-400-400-10 & - & 20   \\
				& SNN-L      & 83.75 $\pm$ 0.15  & 28$\times$28-1000-R28-10 & - & - \\
				& SLAYER     & 85.73 $\pm$ 0.16  & 28$\times$28-500-500-47 & - & 50   \\
				& SLAYER-$U_2$  &  86.62 $\pm$ 0.03  & 28$\times$28-500-500-47 & - & 50 \\
				& BSNN (this work)
				& \textbf{87.51 $\pm$ 0.23} & 28$\times$28-500-500-47 & -0.37 & 50   \\
				\bottomrule
			\end{tabular}
			\begin{tablenotes}
				\item[$\spadesuit$]:-300-300- denotes two hidden layers with 300 spiking neurons, while -800- is one hidden layer with 800 spiking neurons.
				\item[$\heartsuit$]: SNN-EP~\citep{o2019} proposes an implementation for training SNN with equilibrium propagation.
				\item[$\clubsuit$]: -$U_1$ and -$U_2$ indicate the alternating optimization algorithms with parametric control rates sampled from $U_1$ and $U_2$, respectively.
				\item[$\diamondsuit$] : R400 represents a recurrent layer of 400 spiking neurons.
			\end{tablenotes}
		\end{threeparttable}
	}
\end{table*}

\begin{figure}[!htb]
	\centering
	\subfigure[Example 4881.]{
		\begin{minipage}[t]{1\textwidth}
			\includegraphics[width=0.9\textwidth]{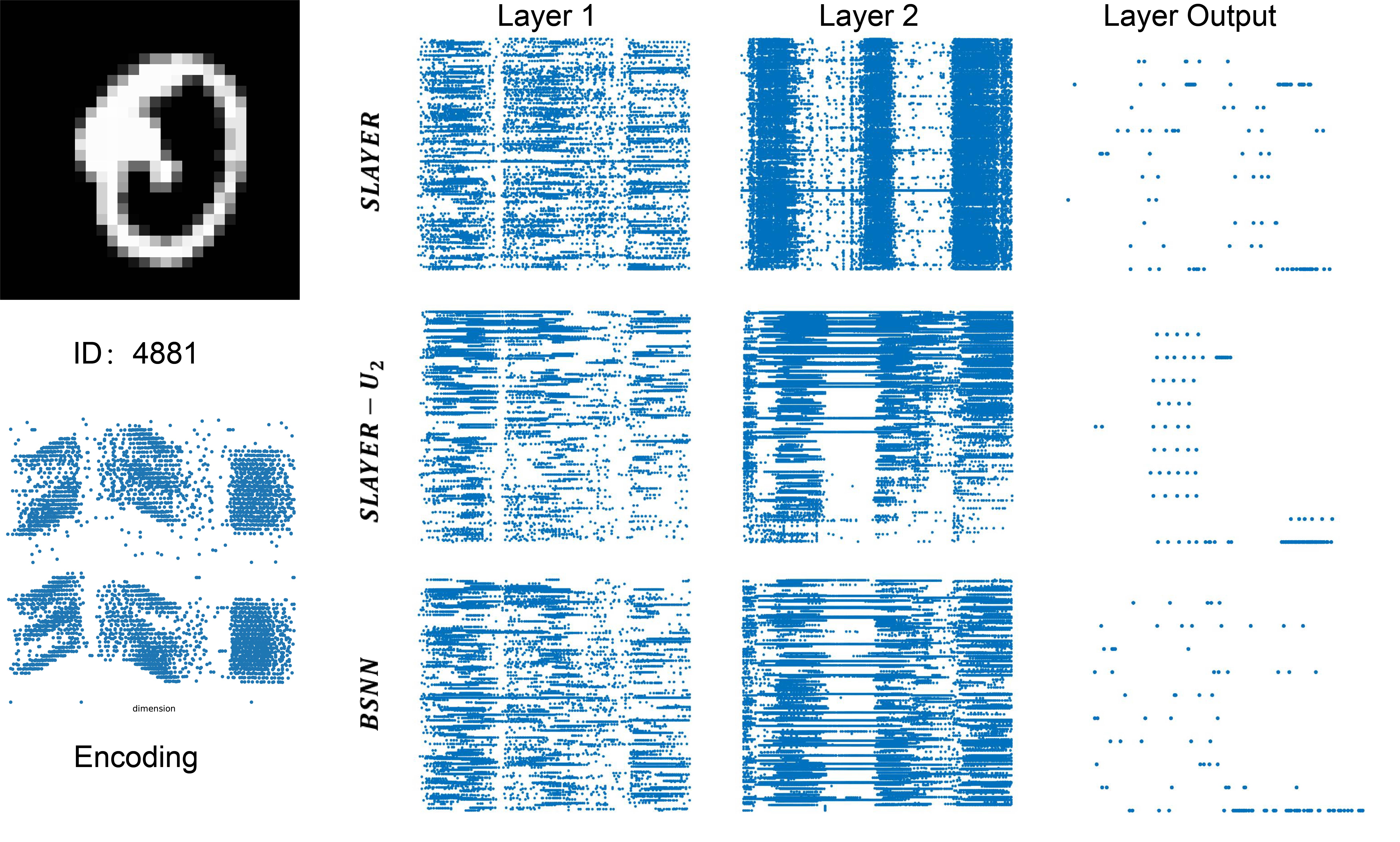}
		\end{minipage}
	}
	\subfigure[Example 6429.]{
		\begin{minipage}[t]{1\textwidth}
			\includegraphics[width=0.9\textwidth]{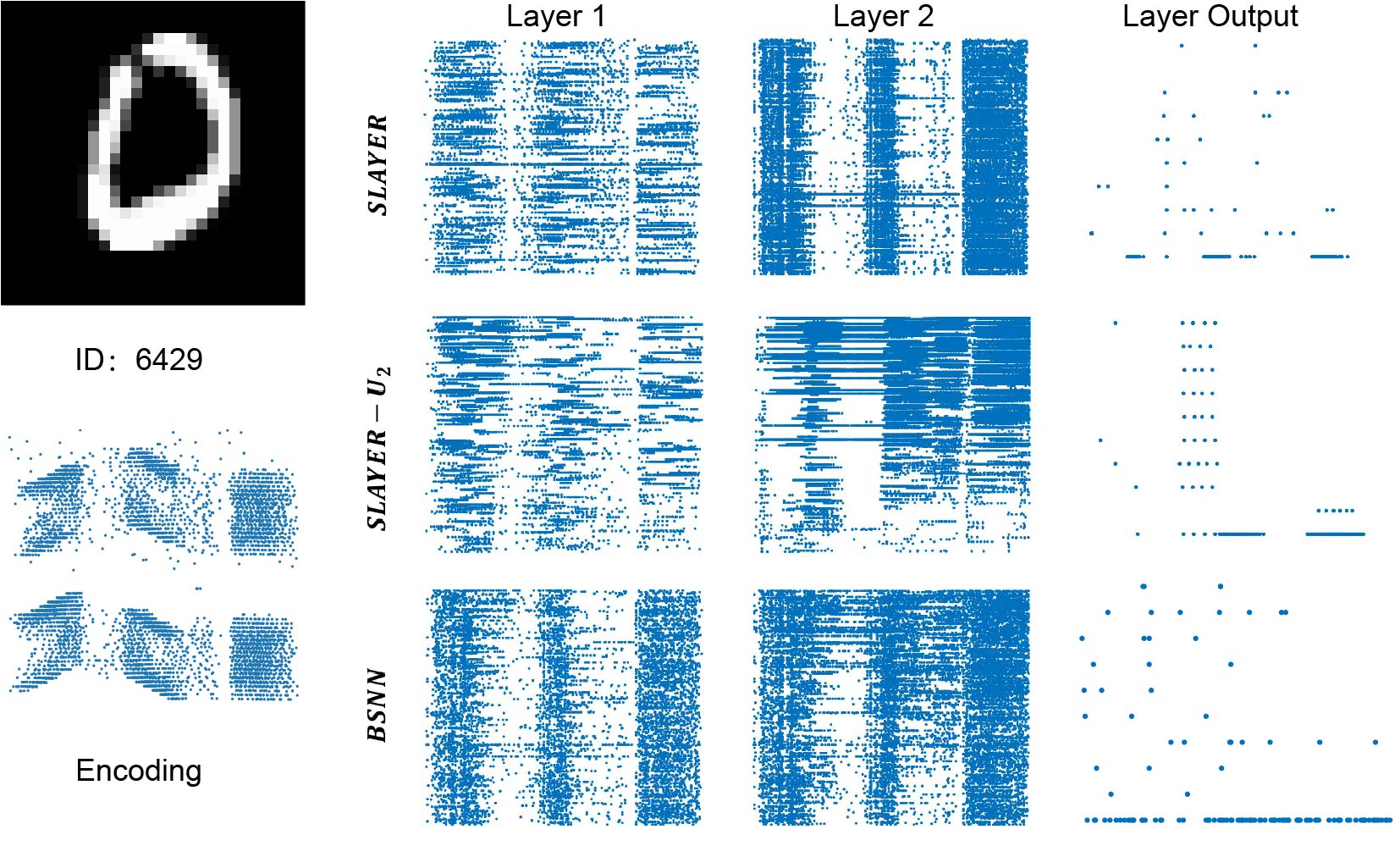} 
		\end{minipage}
	}
	\label{fig:rasters}
	\caption{The spike raster plots of SLAYER, SLAYER-$U_2$, and BSNN on (a) 4881 and (b) 6429.}
\end{figure}
To help understand the learning process of SNNs and highlight the difference between the contenders and BSNN, we illustrate the spike raster plots of SLAYER, SLAYER-$U_2$, and BSNN. We show the neuron excitation snapshots of these three approaches on the 4881-st and 6429-th MNIST testing samples with label 0 in Figure~\ref{fig:rasters}. The horizontal and vertical axes indicate the time interval and the sequence of spiking neurons or encoding channels, respectively. In detail, we pick up a representative instance, the 4881-st one, and then convert this image into a spike sequence using the Dynamic Vision Sensor, as shown in the left two subplots of Figure~\ref{fig:rasters}(a). It is observed that there are two fractures, one around time interval $[80,100]$ and the other around time interval $[200,230]$. The spike raster plots of spiking neurons (in Layer 1, Layer 2, and Layer Output) of SLAYER, SLAYER-$U_2$, and BSNN are successively shown in the right nine subplots. As mentioned in Section 4, a dissipating system would hinder the spike excitation. Thus, we can observe the plots of SLAYER with the following properties. (1) The neurons in the same hidden layer have almost the same firing rates. (2) The plots of Layer 1 and Layer 2 of SLAYER display two distinct fractures. Specifically, Layer 2 even enlarges the fracture. In practice, it is more likely to fall into the situation of ``dead neurons''. (3) In Layer Output, the neuron corresponding to label 0 has no obvious advantage over other neurons in terms of total spiking counts, leading this instance to be incorrectly classified as label 8. In contrast, both SLAYER-$U_2$ and BSNN have indefinite eigenvalues, where negative ones hinder the spike excitation, positive ones promote the spike excitation, and ones with zero are a conservative system. Thus, the developed models can dynamically determine the firing rates of spiking neurons. (1) The firing rates of spiking neurons in the same layer show significant differences. (2) The output spiking neurons relative to wrong labels are suppressed. In contrast, the neuron relative to the correct label is ``encouraged" to fire more spikes and eventually win with a significant advantage.

We also demonstrate the robustness of BSNN to the control rate. This experiment is conducted on the MNIST data set, setting the architecture of BSNN as 28$\times$28-500-500-10. The control rate hyper-parameters are optimized by grid search. For each control rate value, we ran BSNN 5 times, recorded the largest accuracy of each round within 50 epochs, and averaged five accuracy records as its performance. The results plotted in Figure~\ref{fig:curves}(a) show that BSNN is able to perform better than the alternating optimization algorithms in a broad-range setting of control rates. Besides, Figure~\ref{fig:curves}(b) and (c) show the learning curves with $\gamma=-0.21$ on MNIST. As we can see, BSNN converges fast and surpasses the contenders.
\begin{figure*}
	\centering
	\subfigure[]{
		\begin{minipage}[t]{0.315\textwidth}
			\includegraphics[width=1\textwidth]{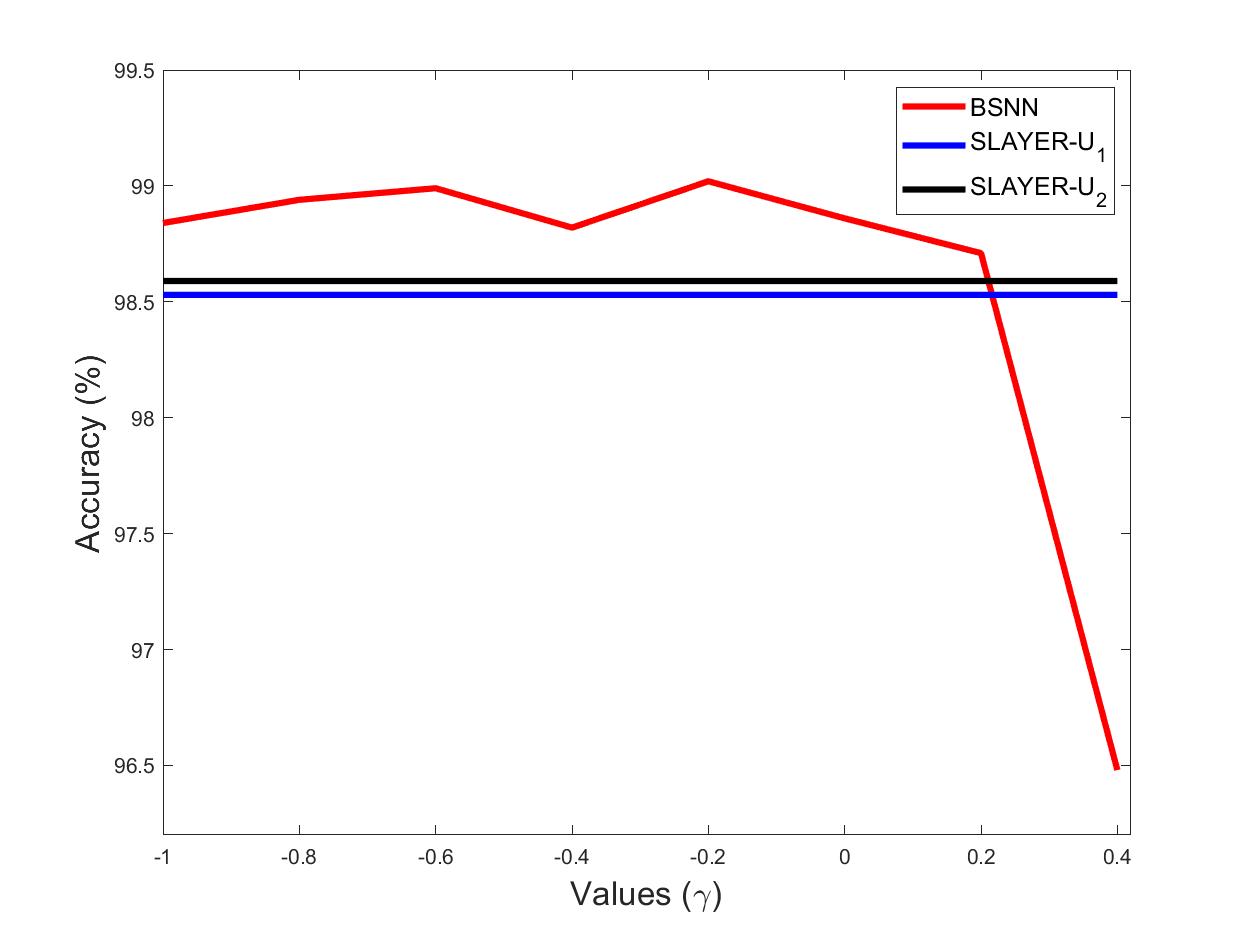}
		\end{minipage}
	}
	\subfigure[]{
		\begin{minipage}[t]{0.315\textwidth}
			\includegraphics[width=1\textwidth]{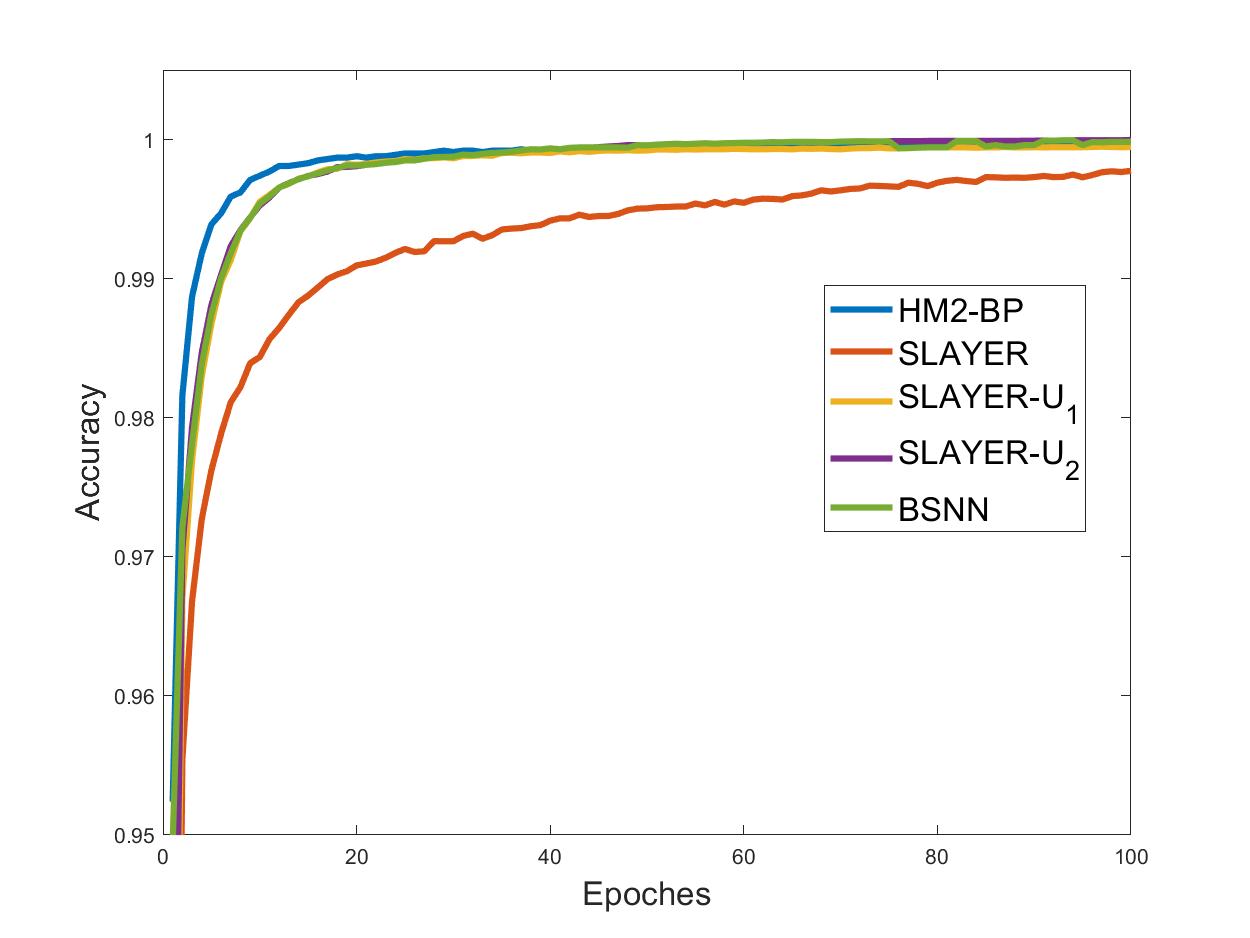} 
		\end{minipage}
	}
	\subfigure[]{
		\begin{minipage}[t]{0.315\textwidth}
			\includegraphics[width=1\textwidth]{Figs/curves_training.jpg} 
		\end{minipage}
	}
	\label{fig:curves}
	\caption{(a) Robustness of BSNN with respect to control rates. The training (b) and testing (c) accuracy curves of the contenders and BSNN on MNIST.}
\end{figure*}

Based on the experimental results and discussions above, we can conclude that BSNN achieves superior performance to the existing SNNs and the improved contenders - approaches based on alternating optimization algorithms. Additionally, the performance of BSNN is considerably robust to the setting of control rates.

\section{Conclusion and Discussions} \label{sec:Conclusion}
In this paper, we investigated spiking neural models and networks from the perspective of dynamical systems. We reveal the dynamical properties of spiking neural models by providing their energy function and conclude that the LIF model is a bifurcation dynamical system, where the control rates are the corresponding bifurcation hyper-parameters. Further, by employing the spiking neural model to enable adaptable eigenvalues, we proposed the \emph{Bifurcation Spiking Neural Network} (BSNN). Compared with the alternating optimization approaches, BSNN tackles the challenge that control rates interact with connection weights in the training procedure, leading to a robust setting of control rates. Besides, BSNN achieves a better accuracy with considerably less computation and storage in supervised classification tasks. The experiments verified the effectiveness of BSNN.

We introduced a mathematical framework for analyzing the spiking neural models, including and not limited to using a Lyapunov-like function to formulate the total energy, revealing the bifurcation characteristics of SNNs, and providing systematic guidance on hyper-parameter setting in SNNs. These results may contribute to the development of SNN-related theories. Besides, we also declare that our work doesn't aim to realize a biological learning phenomenon but to explore some new thoughts on SNNs. In this situation, Equation~\ref{eq:feedforward} that employs the last spikes of adjacent neurons to approximate the mutual promotion only provides a feasible paradigm of implementing dynamic bifurcation neurons. We are interested in scaling up our work.

\section*{Acknowledgment}
This research was supported by the National Science Foundation of China (61921006).
\nocite{zhang2020}.

\bibliography{reference}
\bibliographystyle{plain}

\end{document}